\documentclass[aps,pra,onecolumn,nofootinbib,superscriptaddress]{revtex4-2}
\usepackage{graphicx}
\usepackage{dcolumn} 
\usepackage{bm}      
\usepackage{mathtools}
\usepackage{color}
\usepackage{float}
\usepackage{bbm}
\usepackage{amssymb,amsmath,amsthm,amsfonts}
\usepackage{natbib}
\usepackage{circuitikz}
\usepackage{tikz}
\usetikzlibrary{patterns}
\usepackage{hyperref}
\usepackage{tikz-cd}
\newtheorem{theorem}{Theorem}[section]

\newtheorem{definition}[theorem]{Definition}

\theoremstyle{remark}

\theoremstyle{remark}

\numberwithin{equation}{section}

\begin{document}
\title{Topological understanding of Neural Networks}
\author{Tushar Pandey }
\email{tusharp@tamu.edu}
\affiliation{Department of Mathematics, Texas A\&M University, College Station, TX 77840}

\begin{abstract}
    In this review paper, we look at the internal structure of neural networks which is usually treated as a black box. The easiest and most comprehensible thing to do is to look at a binary classification and try to understand the approach a neural network takes. We review the significance of different activation functions, types of network architectures associated to them, and some empirical data. At the end, we conclude with describing some possible choices of activation functions for different problems and techniques. 
\end{abstract}
\maketitle
\section{Introduction}
One of the prominent questions in deep learning is understanding what happens inside the black box, i.e. the hidden layers. The theme of the paper is to understand what happens to the data when it goes through different layer. There are some other approaches taken, one where each data point is looked at after every layer, usually images. The second one is to understand the boundary manifold and how that changes in the hidden layer. Even though these methods are important, we believe it's more important to look at the transformation of the entire data set as it goes through the hidden layers, and see the representation of the data space in the final layer.
\\
We begin the paper by looking at some smooth activation functions, where width plays an important role. We provide some intuition behind selecting neural networks with different architectures. We point out some possible errors in considering such methods and the time complexity that comes with it. 
\\
In the third section, we look at a comparison between smooth and non-smooth activation functions, along with changing width and depth of the network. We try to answer a widely asked question, what makes ReLU better than other activation functions in practice? \cite{NH}, \cite{MHN}, \cite{GBB} We begin with a simulated dataset, where the topology is known, in order to understand the changes that could take place in an actual manifold. Under different conditions, these changes are measured. Once there is an idea of the change, it is verified on real world data in high dimensions. Due to computational power constraints, some parameters of the architecture are not adjusted from the simulated data to real world data with a large difference in dimensions.
\\
After reviewing these methods, we draw some conclusions from both the approaches and provide with experiments to extend the results to different architectures in order to improve the understanding of the black box.
\\
For most of the paper, we will consider the case where the data has two classes. All the definitions are provided at the end of the paper in Appendix.

\section{Smooth Activation functions and change in topology}
In this section, we look at the change in topology when the activation function is smooth. This section is based on the work of \cite{Olah} 
\\
Consider an architecture where there is no hidden layer. Let's assume the data forms two lines as described in fig \ref{fig:No hidden Layer}. 
\begin{figure}[h]
    \centering
    \includegraphics[width = 0.2\textwidth]{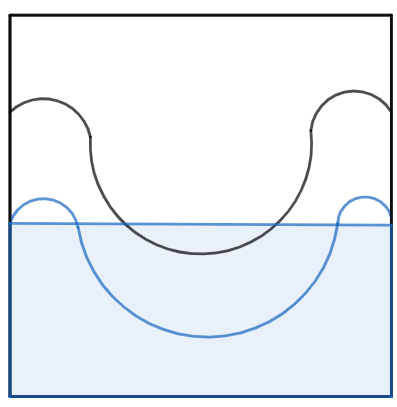}
    \caption{ No hidden Layer}
    \label{fig:No hidden Layer}
\end{figure}

Since there is no hidden layer, the output layer is a linear function, therefore the neural network tries to classify the data by separating it through a straight line (or a hyperplane in case of higher dimensions). In this example, it can not properly classify the data set as no straight line can possibly distinguish the two classes completely. Note that there is no activation function involved so far.
\\
Now consider the case with one hidden layer. We look at the $tanh$ function and see the boundary line between the two classes of data in fig \ref{fig:one-hidden-layer}.
\begin{figure}[h]
    \centering
    \includegraphics[width = 0.2\textwidth]{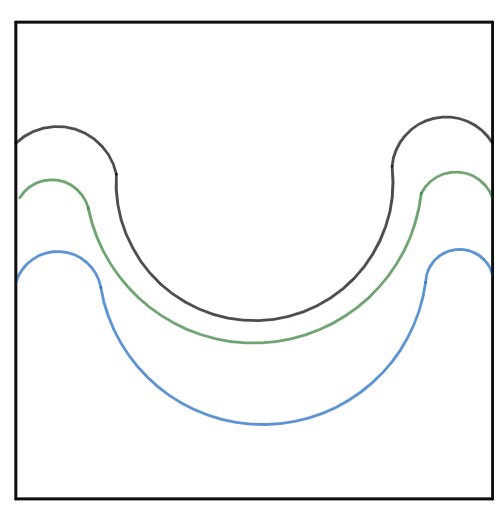}
    \caption{  "tanh" activation function with one hidden layer. The green line is the boundary}
    \label{fig:one-hidden-layer}
\end{figure}
\\
This separation is seen in the actual data space, by which we mean the original way the data is represented. However, internally, the network doesn't try to change the boundary shape from a line to a curve, rather, it changes the shape of the data and then fits a linear separating boundary. More precisely, the ambient space of the data changes after applying the activation functions and therefore it's easier for the neural network to construct a linear boundary. An example of the same is demonstrated in fig \ref{fig:hidden_layer_linear_separation}.
\begin{figure}[h]
    \centering
    \includegraphics[width = 0.2\textwidth]{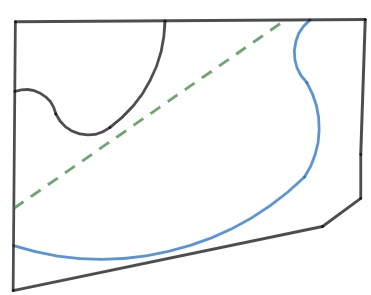}
    \caption{The data space is transformed and the final boundary (in green) is linear.}
    \label{fig:hidden_layer_linear_separation}
\end{figure}
\\ 
For a smooth activation functions, this change in the data space is a homeomorphism. 
A natural question that arises:
\\
\textit{"Is one hidden layer enough to change the data space (manifold) in order to separate two classes via a linear boundary for smooth activation functions?"}
\\
Answer: No. In fact, depth is not sufficient to decide whether or not two data sets are separable. The depth of the network corresponds to the number of transformations. Since the functions are homeomorphisms, they can't change some topological properties of the space. One such property of a space is the homology (or the Betti numbers). 
Back to neural networks; if one has an annulus as one class of data points and in the hole, there is a cluster corresponding to a different class, no homoemorphism in $\mathbb R^2$ can distinguish between these classes for a smooth activation neural network. Since it's a homeomorphism, $\beta_1$ will remain 1, which means, no linear boundary can separate the two classes completely.

Question: \textit{How about changing the width?}

Answer: yes! It will work. Width corresponds to an embedding in a higher dimensional space, where one can lift up the class inside the hole and by a hyperplane, separate them. This means, if we know the minimum dimension of the ambient space where the data can be embedded properly, we can use a neural network to completely classify the dataset.
\begin{theorem}
For the data space described above, any neural network with a smooth activation function, width 1 or 2 is not enough to completely classify the two data classes regardless of the depth.
\end{theorem}
\begin{proof}
For a smooth activation function, hidden layer corresponds to a homeomorphism, which means $\beta_1 = 1$ after all the hidden layers. The last layer is a linear transformation, therefore conserves $\beta_1$ as well. In order to divide the space by a linear boundary, $\beta_1$ needs to be 0. Therefore, a contradiction.
\end{proof}
Note that this was true only because of the smoothness of the activation function. Which means, if the activation function is not smooth (e.g. ReLU), the Neural network can form a linear boundary to distinguish the two classes.
For a dataset with n dimensional points, they can be embedded in ambient space of dimension 2n+2, such that a linear (hyperplane) boundary can separate them. Therefore, if the hidden layer has width $\geq 2n+2$, then the data can be separated.
\begin{theorem} \cite{Olah}
There is an ambient isotopy between the input and a network layer’s representation if: a) W isn’t singular, b) we are willing to permute the neurons in the hidden layer, and c) there is more than 1 hidden unit.
\end{theorem}
So far, we have a way to find a nice boundary to potentially get $100 \%$ accuracy on the training data. This however, does not mean it's the actual separation. The author \cite{Olah} also suggests that based on some empirical results, changing the last layer from softmax to KNN, the accuracy increases. This approach concludes that for smooth activation, the width needs to be wide enough and there is a need of sufficient (but small) depth in order to not force the network to get stuck at a local minima. 
\section{Smooth vs non-smooth activation functions}
We follow \cite{NZL} for most of the section. Once again, the data is divided into two classes $M = M_a \cup M_b$. The first conclusion of this section describes ReLU outperforming smooth activation function, namely tanh, and the second one is describing the depth instead of the fact that shallow networks can approximate most functions pretty well. \cite{NZL} have performed the analysis on real and simulated data to conclude such observations. The topological work is done through TDA, based on the idea of persistent homology originally inspired by \cite{C1,C2,CZ}
\\
The interesting aspect of this study is that instead of looking at the change in each data point, the overall change in the data shape is observed, which provides a better insight. The focus of the study is to look at the change in Betti numbers and how it is affected in the training.
\begin{figure}[h]
    \centering
    \includegraphics[width = 0.95\textwidth]{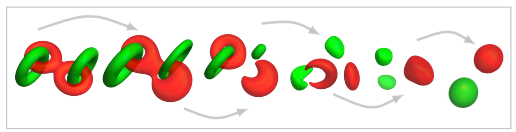}
    \caption{\cite{NZL} Change in Betti numbers:
    \\
    $\beta(red): (1, 2, 0) \rightarrow (1, 2, 0) \rightarrow (2, 1, 0) \rightarrow (2, 0, 0) \rightarrow (1, 0, 0) \rightarrow (1, 0, 0);
    \\
    \beta(green): : (2, 2, 0) \rightarrow (2, 2, 0) \rightarrow (2, 1, 0) \rightarrow (2, 0, 0) \rightarrow (2, 0, 0) \rightarrow (1, 0, 0).$ }
    \label{fig:progression_betti}
\end{figure}
\\
In fig \ref{fig:progression_betti}, the actual network changes the shape of the data like demonstrated. In order to unlink the components, it has to break the topology, and change the betti numbers. Generally speaking, for binary classification problems, the idea is to decrease the betti numbers, s.t. $\beta_1(M_a) = 0 = \beta_1(M_b)$ and $\beta_2(M_a) = 0 = \beta_2(M_b)$, whereas $\beta_0(M_a) = 1 = \beta_0(M_b)$. Note that the last condition is not very strict. Even if $\beta_0(M_a)>1$ while other $\beta_j(M_a) = 0$ for $j\geq 1$, it's still a good enough classification. In fact, if $\beta_0 > 1$, the suggests possibility of another class (or subclass), providing more insights about the data.
\\
Out of some of the obstructions this hypothesis possesses, one of them is that the data usually doesn't come in such nice form. The data will more likely be in a point cloud form with some noise. But, that's where persistent homology comes into picture. Presence of small noise does not change the effective $\beta_j$ for the point cloud data. If the data does come from some sampling of a manifold structure, then persistent homology recovers the $\beta_j$ almost precisely. The authors of the paper have not only looked at the topological changes for simulated data, but also for real world data including images. 
\\
Some important questions which will be answered here are:
\\
1) Why does ReLU perform better than others empirically? 
\\
2) Are these topological changes observed through this method robust?
\\
3) Why do deep neural network work better than shallow ones even after the approximation theorems?
\\
\subsection{Topology}
We try to look into topological complexity of the manifold and the generalization gap of the dataset which measures the difference in test and training accuracy of the model. The Generalization gap is defined as 
\begin{align*}
    GG(X_{train}, X_{test},m) = Acc(X_{train},m) - Acc(X_{test},m)
\end{align*}
where m is the model.
\\
For piece wise linear function the upper bound for topological complexity is given in terms of linear regions. However, the number of linear regions determined through the training set is not stable under small perturbations \cite{ZNL}. 
Instead of looking at the decision boundary at different stages, the data space transformations are the object of interest here. In order to understand the transformations, one can look at the change in Betti numbers. Furthermore, instead of looking at the Betti number of the entire Manifold, it's sufficient to look at the Betti number progression of each component. In practice, it's usually difficult to compute the homology for a point cloud. 

The standard practices in Topological Data Analysis (TDA) is i) Discard outliers, noise ii) Construct $\epsilon-$Vietoris-Rips complex iii) Simplify VR Complex without changing the topology. The topological structure can be altered by i) and ii) based on choices of noise reduction/smoothing and $\epsilon$ value in VR complex. 
\subsection{Setup}
\begin{enumerate}
    \item The problem assumed is a binary classification problem, $M = M_a \cup M_b$ with an additional assumption, $\inf \{ ||x-y||: x \in M_a, y \in M_b \} >0 $.
    \item  For the simulated dataset, we know the manifold, so to generate the point cloud a large sample is selected uniformly and densely. The neural network chosen in feed forward, with depth l.
    \item The network function is $v: s \circ f_l \circ  \dots \circ f_1 :  \mathbb R \rightarrow [0,1] $.s is the score function, $s: \mathbb R^{n_{l+1}} \rightarrow [0,1] $. 
    \item Let $n_j $ denote the width of the layer j. We let $n_1 = d, n_{l+1} = p$ and $v_j : f_j \circ \dots \circ f_1$.
    \item  The simulated data is non-realistic with complicated topology in low-dimension. Real world data is in high dimensions but most likely simpler in terms of entangled classes. For the simulated data we look at large $\beta_0$, non-zero $\beta_j$ for $j \geq 1$ and large Topological complexity.
    \item  The model is trained to near-zero generalization error. 
\end{enumerate}
 Spoiler alert: with sufficient depth, the last layer maps $v(T_a), v(T_b)$ on the opposite ends of $[0,1]$.
There are some challenges moving from simulated to real data. The topology is not known so the persistent homology is harder to compute at each layer. 
\\
The work is done on three different dataset which can be seen in the following figure \cite{NZL}. The red part of the data is $M_b$ and the green part is $M_a$.
\begin{figure}[h]
    \centering
    \includegraphics[width = 0.75\textwidth]{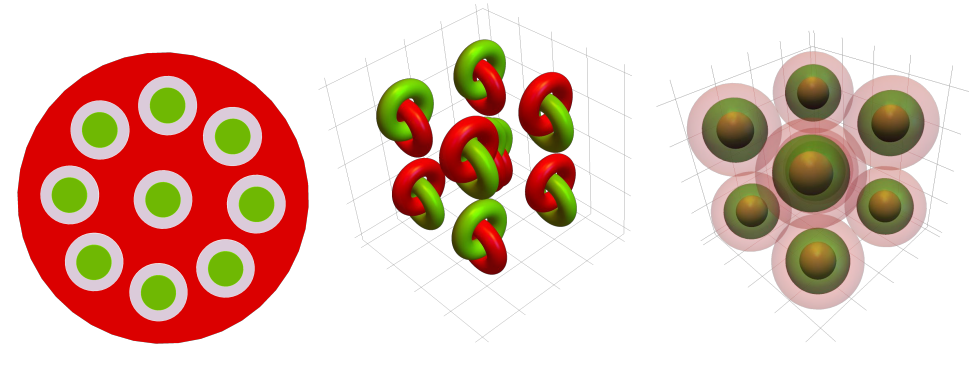}
    \caption{ \cite{NZL} a) $\beta(M_a) = (9,0)$, $\beta(M_b) = (1,9)$, b)  $\beta(M_a) = (9,9,0)$, $\beta(M_b) = (9,9,0)$,\\ c)  $\beta(M_a) = (9,0,9)$, $\beta(M_b) = (18,0,9)$}
    \label{fig:three_dataset}
\end{figure}

\subsection{Training}
\begin{enumerate}
    \item Different activation functions: tanh, ReLU, leaky ReLU
    \item Different depth from 4-10
    \item Different width between 6-50
    \item Criss entropy categorical loss
    \item ADAM with 18000 epoch
    \item $\eta = 0.02-0.04$ with exponential decay $\eta^{t/d}$, $d = 2500$ and t is the epoch.
    \item For bottleneck architecture (narrow width in middle), d = 4000, $\eta = 0.5$
    \item Score = softmax function
    \item Metric $\delta_k$ for $VR_{\epsilon}$ Complex is the graph geodesic distance on K-nearest neighbours. $\delta_k(x_i,x_j) = $ minimum no. of edges between them in knn graph. It preserves connectivity and normalizes distance.
    \item Two hyperparameters, $k$ and $\epsilon$. The persistant homology is used through filtered complex w.r.t. $\delta_k$ at $\epsilon$. Find $k^*$ with $\epsilon=1$ first by $\beta_0(VR_{k^*}) = \beta_0(M)$. Then find $\epsilon^*$ by equating $\beta_1$ and $\beta_2$.
\end{enumerate}
\subsection{Results}
\begin{enumerate}
    \item Clear decay in $\beta_0$ across all possible neural network architectures. It's slower in tanh, faster in Leaky ReLU and fastest in ReLU.
    \item Width: 
    \begin{itemize}
        \item Narrow: 6 neurons each layer, changes topology faster
        \item Bottleneck: one of the middle layer has 3 neurons while other have 15, sudden change in topology at bottleneck.
        \item Wider: 50 neurons each, smoother reduction in topological complexity.
    \end{itemize}
    \item depth:
    \begin{itemize}
        \item For highly entangled classes, more depth required for $100\%$ accuracy. Low depth makes it difficult to get the accuracy we want
        \item Faster reduction in Topological complexity in the final layers.
        \item Sometimes, there's non-required simplification for deeper neural networks.
    \end{itemize}
    \item For wide enough layers, $\beta_1(M) = 0$ and $\beta_2(M) = 0$ is not required. Shallow networks sometimes forces this but deeper ones don't, therefore preserving more structure.
\end{enumerate}
\subsection{Graphical Results}
\begin{figure}[h]
    \centering
    \includegraphics[width = 1\textwidth]{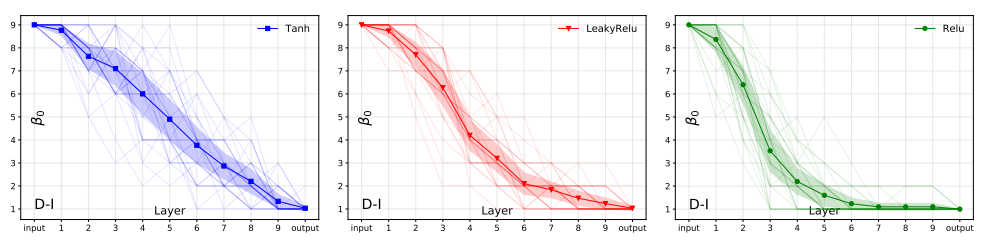}
    \caption{\cite{NZL}  Average change in $\beta_0$ for different activation functions for Dataset I. The dark line is the average whereas the shaded regions are results of multiple simulations.}
    \label{DI}
\end{figure}
\begin{figure}[h]
    \centering
    \includegraphics[width = 0.95\textwidth]{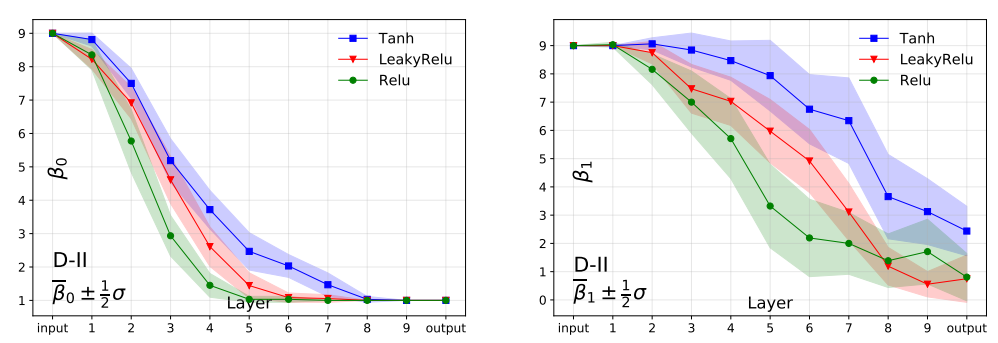}
    \caption{\cite{NZL}  Change in $\beta_0$ and $\beta_1$ for Dataset II for different activation functions.}
    \label{DII}
\end{figure}
\begin{figure}[h]
    \centering
    \includegraphics[width = 0.9\textwidth]{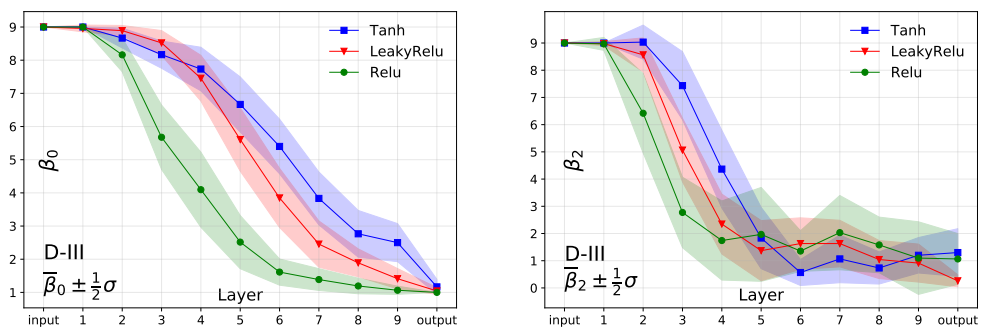}
    \caption{ \cite{NZL} Change in $\beta_0$ and $\beta_2$ for Dataset III for different activation functions.}
    \label{DIII}
\end{figure}
\begin{figure}[h]
    \centering
    \includegraphics[width = 0.9\textwidth]{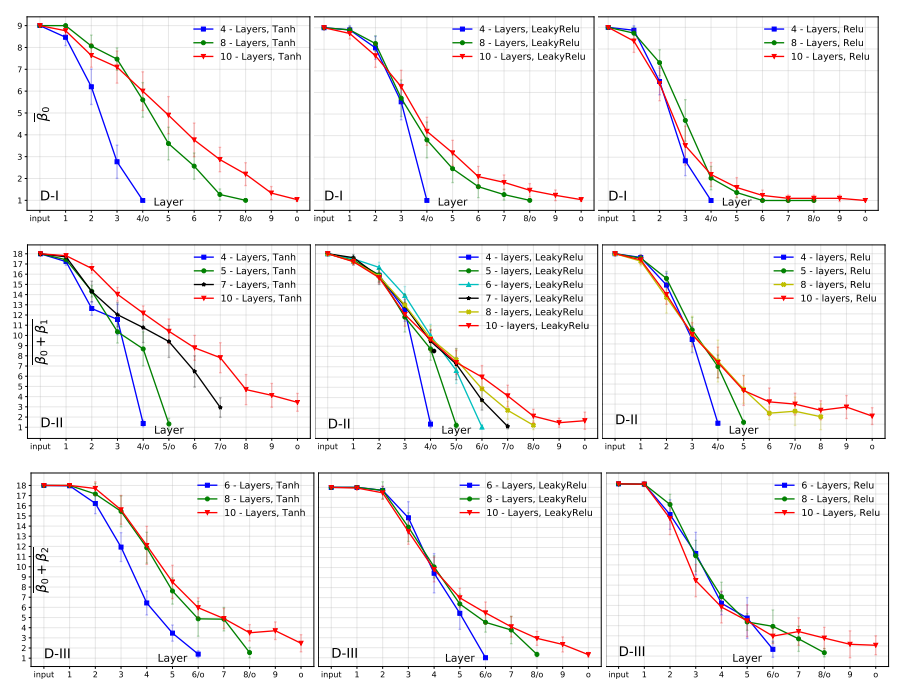}
    \caption{ \cite{NZL} Topological complexity for varying network depth for all three datasets.}
    \label{Varying depth}
\end{figure}
\subsection{Real Dataset}
These properties were verified for MNIST, HTRU2, UCI Banknotes and UCI sensorless drive datasets. Since the dimension of these datasets are high, the persistent homology can not be performed as efficiently and needs to be performed on every layer. The generalization gap is relaxed, from $2-5 \%$. The width and depth of the networks are fixed.
\\
MNIST: $\mathbb R^{784}$, for this dataset, the topological observations are made on top 50 principal components, 
HTRU2 Dataset: $\mathbb R^8$, 
UCI Banknotes: $\mathbb R^4$, 
UCI drive: $\mathbb R^{49}$
\\
The table of results are available in the appendix. The observation as expected were: 
\begin{enumerate}
    \item The topological complexity is reduced overall. The network tries to reduce $\beta_0$ to 1 and other $\beta_j$ to 0. 
    \item Smooth activation function reduces the topological complexity slower than non-smooth ones with ReLU performing the most simplification.
    \item ReLU adds a folding to the data space but not entirely. This is immediate from the definition of ReLU or the absolute value function. 
    \item More layers makes it easier to train the model by "taking it's time simplifying the space step by step".
\end{enumerate}
\section{Future work}
While the possibility of research in this direction is never ending, some possible options for near term research experiments are: 
\\
\subsection{Different activation functions in different layers}
For training a model on multiple data sets or for training a dataset with composite architectures, setting the activation function for the initial layers as a smooth one, and applying ReLU or Leaky ReLU at the end could better transform the data. The idea behind this is that smooth activation function preserves the structure in the same dimension, but embedding it in higher dimension would reduce some complexity with some structures still being preserved. 
\subsection{Relationship between time complexity and robustness against noise}
There are some results about different linear folding because of ReLU which are not stable under noise. There seems to be a trade-off between oversimplification and time. It'll be interesting to relate the stability of folding under a network architecture with non-ReLU activation functions in some layers as well. 
\\
\subsection{Different neural network architectures}
Different architectures for the same dataset, e.g. CNN, simple feed forward network and ResNet should be compared in terms of changes in topological complexity.

\section{References}

\bibliography{main.bib}

\appendix
\section{Definitions}
The definitions listed here are more intuitive than formal for the sake of understanding for a wider audience.
\begin{definition}
\textbf{Homeomorphism}: A function $f: M \rightarrow N$ is considered to be a homeomorphism if f is continuous, f is a bijection (one-one and onto) and $f^{-1}$ is continuous.  
\end{definition}
\begin{definition}
\textbf{Connected Component}: A space M is called connected if it's not a disjoint union of two components. The number of connected component is the cardinality of the components that are connected. 
\end{definition}
\begin{definition}
\textbf{Homology}: Let M be a manifold (locally homeomorphic to $\mathbb R^n$ for some $n \in \mathbb N$)
\\
The $0^{th}$ homology, $\mathbb H_0(M)$ is defined to be the number of connected components of M.
\\
The first homology, $\mathbb H_1(M)$ is related to how many distinct circles (non-trivial) can one embed in the manifold M. Based on the number and the relationship between them, a group with integer coefficients are associated to it.
\\
Similarly one can also define $\mathbb H_2(M)$ by relating how many spherical holes are present in the manifold.
\\
Betti numbers denote the rank rank of the homology group. e.g. $\beta_0 = \mathrm{rank}(H_0(M)), \beta_1 = \mathrm{rank}(H_1(M)), \beta_2 = \mathrm{rank}(H_2(M))$
\end{definition}
\begin{definition}
    \textbf{Persistent Homology}: Suppose there are some data points in an ambient space. At each point, start with a ball of radius $\epsilon$. If two such balls intersect, add an edge. If there are three balls which intersect pairwise and together as well, draw triangular face. In this way, based on intersections, one builds a simplicial complex from data set by slowly increasing the radius $\epsilon$ of the ball around each data point. 
    \\
    At each radius value, the homology of the data set is computed. The homology which remains the same for the longest time, or in other words, persists, is considered the homology of the data set (manifold).
\end{definition}
\begin{definition}
    \textbf{Topological Complexity}: It's $\sum_{j=1}^d\beta_j$ for a d-dimension manifold. 
\end{definition}
\begin{definition}
    \textbf{Simplicial Complex}: Set composed of points, edges, faces, tetrahedrons and so on, such that all faces of these elements are in the complex as well as, if two elements intersect, they intersect at one face (corresponding to the object)
\end{definition}
\begin{definition}
    \textbf{Vietoris-Rips Complex}: It's the simplex formed in the definition of persistent homology for different $\epsilon$ radius values.
\end{definition}

\section{Examples}
Example of Homology:
\begin{figure}[h]
    \centering
    \includegraphics[width = 0.4\textwidth]{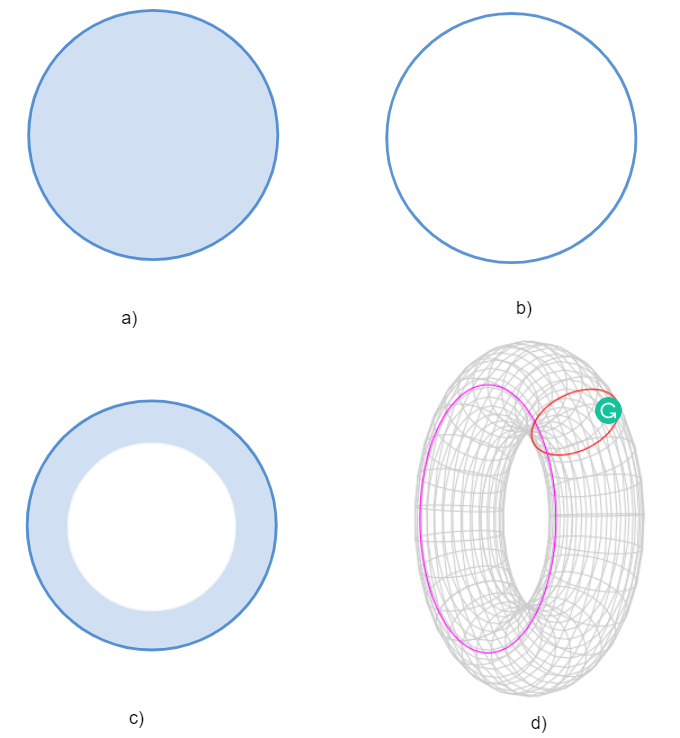}
    \caption{The images are : a) Disk b) Circle c) Annulus d) Torus
    \\
    $\mathbb H_0$ for each of the four parts is $\mathbb Z$ because there is one connected component.
    \\
    a) $\mathbb H_1(D^1) = {e}$ which mean trivial. Any small circle inside the disk can be contracted to a point while being inside the disk at all times.
    \\
    b) $\mathbb H_1(S^1) = \mathbb Z$ because there is one circle. 
    \\
    c) $\mathbb H_1(Annulus) = \mathbb Z$. Two types of circles are possible, one which goes around the hole and another one which is small, completely inside the blue region. The smaller one can shrink to a point, but the other one can not be. Therefore, it is $\mathbb Z$
    \\
    d) $\mathbb H_1 (\mathbb T^2) = \mathbb Z \times \mathbb Z$. There are two non-trivial distinct circles, the red one and the purple one. They can not be continuously deformed into one another without leaving the surface of torus.}
    \label{fig:my_label}
\end{figure}
\begin{figure}[h]
    \centering
    \includegraphics[width = 0.95 \textwidth]{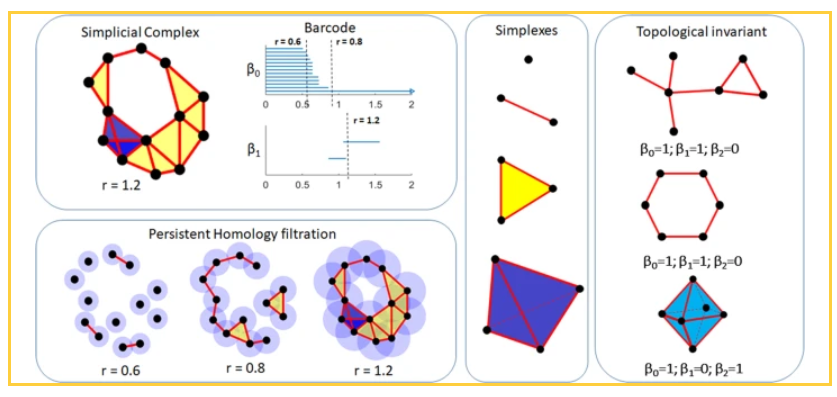}
    \caption{\cite{AMXM} Persistent homology filtration, persistent barcodes based on the radius and the homology of the simplicial complex, topological invariants}
    \label{fig:persistent homology}
\end{figure}
\begin{figure}[h]
    \centering
    \includegraphics[width = 0.5\textwidth]{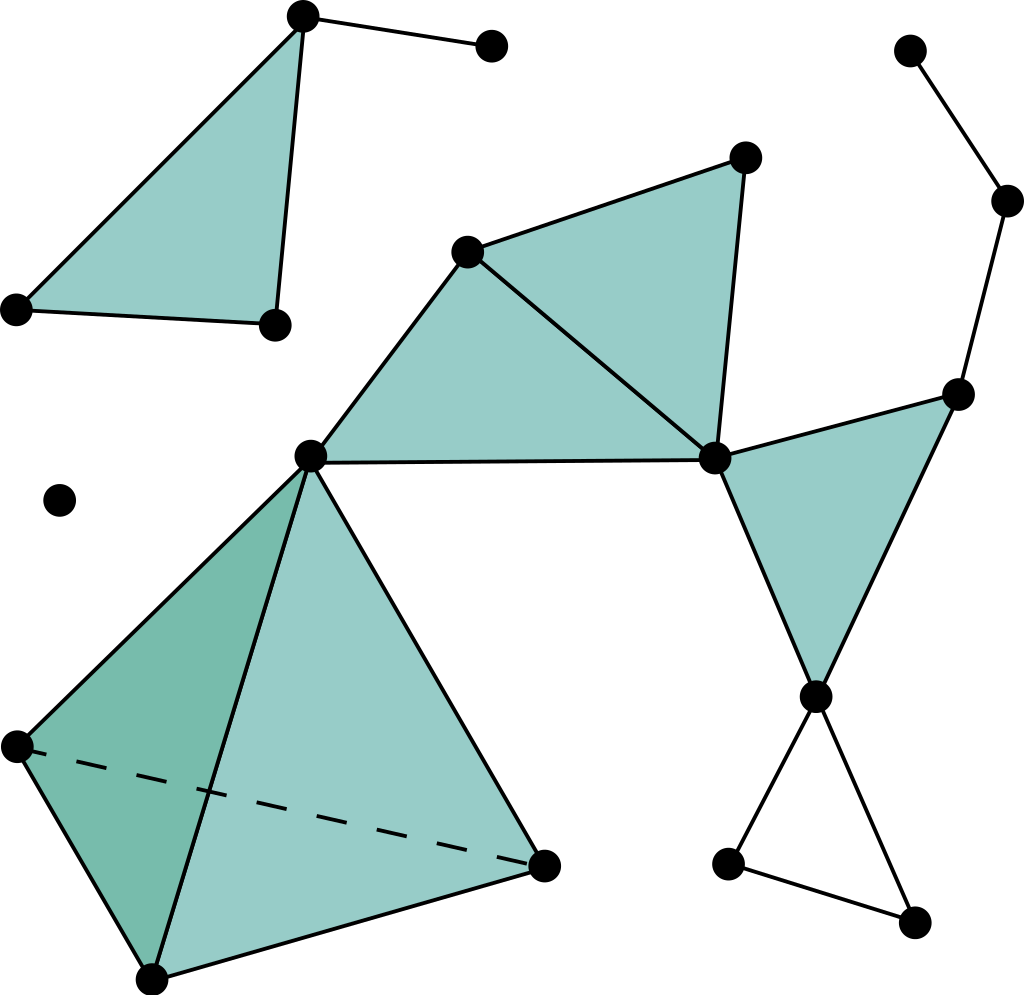}
    \caption{3D simplicial complex}
    \label{fig:simplicial_complex}
\end{figure}

\end{document}